\documentclass[letterpaper, 10 pt, conference]{ieeeconf} 
\IEEEoverridecommandlockouts                              
\usepackage[utf8]{inputenc}
\usepackage[T1]{fontenc}
\usepackage{amsmath,bm}
\usepackage{amsfonts}

\usepackage{algorithm}
\usepackage{algpseudocode}

\usepackage[dvips,frame,rotate,import]{xy} 

\usepackage{array}
\usepackage{tabularx}
\usepackage{multirow}
\usepackage{amsthm}
\let\labelindent\relax  
\usepackage{enumitem}
\usepackage[english]{babel}
\usepackage{cite}

\usepackage[colorlinks=true,linkcolor=blue]{hyperref}
\usepackage{threeparttable} 
\usepackage{multirow}
\usepackage{graphicx}
\usepackage[font=footnotesize]{caption}
\usepackage{subcaption}
\usepackage{titlesec}
\usepackage{nth}
\usepackage{cleveref} 
\usepackage{lipsum}  
\usepackage[dvipsnames]{xcolor}
\usepackage{colortbl}
\definecolor{light-gray}{gray}{0.93}
\definecolor{semilight-gray}{gray}{0.8}

\usepackage{amssymb}  

\newtheoremstyle{defiStyle}  
  {3pt} {3pt} {\normalfont} {} {\bfseries\itshape} {:} { }  
  {\thmname{#1} \thmnumber{#2}\textit{\thmnote{ (#3)}}}  

\theoremstyle{defiStyle}  
\newtheorem{defi}{Definition}  

\theoremstyle{thmStyle}
\newtheorem{thm}{Theorem}  
\newtheorem{lemma}{Lemma}  
\newtheorem{remark}{Remark}  
\newtheorem{prob}{Problem}

\title{\LARGE \bf 
GJK-CBF: Control Barrier Functions for Convex \\ Rigid Body Collision Avoidance on SE(3)
}

\author{Yi-Hsuan Chen$^{1}$, Shuo Liu$^{2}$, Wei Xiao$^{3}$, Michael Otte$^{1}$, Calin Belta$^{4}$
\thanks{$^{1}$Y. Chen and M. Otte are with the Department of Aerospace Engineering, University of Maryland, College Park, MD, USA. {\tt \{yhchen91, otte\}
@umd.edu}}
\thanks{$^{2}$Shuo Liu is with the Department of Mechanical Engineering, Boston University, Boston, MA 02215, USA. {\tt liushuo@bu.edu}}
\thanks{$^{3}$W. Xiao is with the School of Electrical and Electronic Engineering, Nanyang Technological University, and with M3S, SMART, Singapore {\tt wei.xiao@ntu.edu.sg}, {\tt weixy@mit.edu}}
\thanks{$^{4}$C. Belta is with the Department of Electrical and Computer Engineering and Department of Computer Science, University of Maryland, College Park, MD, USA. {\tt calin@umd.edu}}
}

\newcommand{\gjkcbf}{\mbox{GJK-CBF}}

\newcommand{\yihsuan}[1]{\textcolor{red}{[YiHsuan: #1]}}
\newcommand{\calin}[1]{\textcolor{blue}{[Calin: #1]}}

\newcommand{\fixed}[1]{\textcolor{black}{#1}}

\newcommand{\unom}{\overline{\bm{u}}}
\newcommand{\vnom}{\overline{\bm{v}}}
\newcommand{\wnom}{\overline{\bm{\omega}}}

\newcommand{\removeForSpace}[1]{}

\begin{document}
\maketitle
\begin{abstract}  
Collision avoidance among convex bodies is a fundamental problem in robotics. Control Barrier Functions (CBFs) provide a practical framework for real-time safety filtering due to their computational efficiency. For general convex bodies, exact separation measures, such as distance or scaling factor, are typically computed through optimization. Existing CBF formulations often obtain the required gradient via differentiable optimization (diffOpt), adding computational overhead. In contrast, we leverage the Gilbert-Johnson-Keerthi (GJK) algorithm to obtain the current witness pair---the pair of points realizing the minimum distance or penetration depth---and formulate a CBF, termed \gjkcbf, whose gradient is constructed directly from the relative rigid-body motion of the witness pair, {\it without resorting to diffOpt}. This formulation applies to both 2D and 3D environments across a broad class of convex body pairs, provided that at least one in each one-to-one interaction is strictly convex. The proposed \gjkcbf~is validated in various scenarios, including multi-robot position swapping in both 2D and 3D, navigation through a vertical slit in 3D, and its applicability to manipulators. The results demonstrate collision-free motion across all scenarios while reducing the conservativeness introduced by geometric approximations, particularly in narrow environments. A video of all simulations is available at {\url{https://youtu.be/eDvscHhGIj0}}.
\end{abstract}

\section{Introduction}
One of the key difficulties in robotic navigation is to simultaneously satisfy system kinematic/dynamic constraints, safety requirements, and actuation limits. The problem is further complicated when collision avoidance must explicitly account for the geometry of both the ego robot and nearby robots/obstacles. For each pair of bodies, referred to as a collision pair, the separation distance depends on their relative configuration, i.e., position and orientation, and is nontrivial to evaluate and differentiate. The notion of configuration obstacles \cite{lozano-perez_spatial_1983} was proposed to account for robot geometry via Minkowski operations and has been widely used in motion planning \cite{4M, LaValle.Planning.06, Lien.RSS08}. While the same geometric representation underlies classical collision-detection algorithms such as Gilbert-Johnson-Keerthi (GJK) \cite{gjk} and the Expanding Polytope Algorithm (EPA) \cite{epa} \fixed{to characterize distance proximity, these methods are not designed for safety control synthesis.}
\begin{figure}[ht!]
	\centering
    \includegraphics[width=0.92\linewidth]{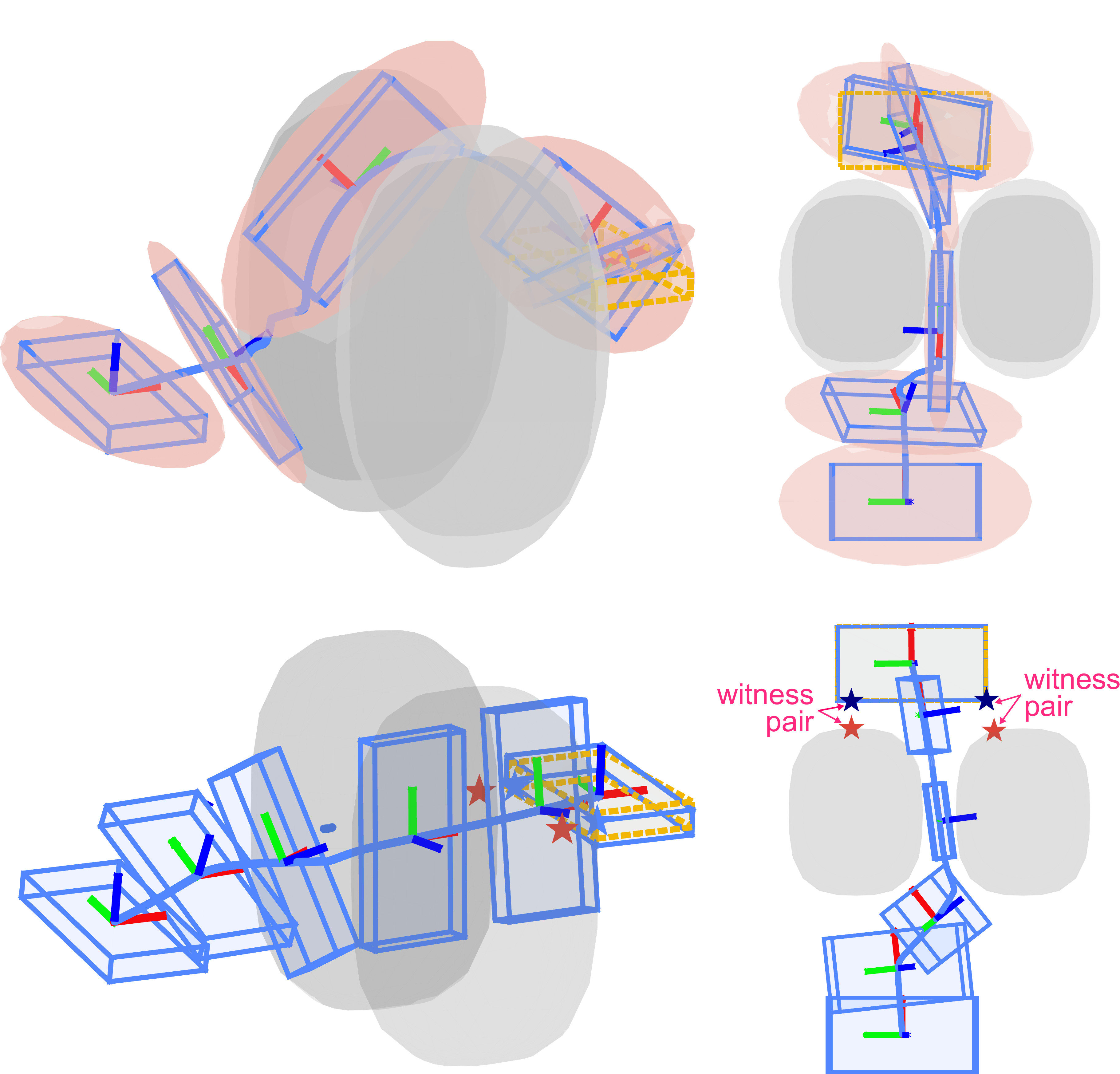}\vspace{0.2em}
	\caption{Side \textbf{(left)} and top \textbf{(right)} views of a polytopic robot (blue) traversing a vertical slit in 3D. \textbf{(Top)} Conservative enclosing-ellipsoid approximations render the direct passage through the slit infeasible, forcing the robot to slightly detour above the obstacles \cite{Funada.etal.TCST24}.
    \textbf{(Bottom)} In contrast, the proposed \gjkcbf~operates directly on the original convex geometries and successfully guides the robot through. Its construction and gradient computation require only the pair of closest points---{witness pair}---between two convex bodies {\it without using differentiable optimization}. Only the initial and goal poses are given, with no intermediate poses or waypoints.
    }\label{fig: teaser}
\end{figure}

Control Barrier Functions (CBFs) \cite{Wieland.Allgower.IFAC07}, on the other hand, provide a systematic approach to synthesizing safe controllers. CBF safety constraints can be implemented on-the-fly, either jointly with Control Lyapunov Functions (CLFs) \cite{Ames.etal.ECC19} or as minimally invasive safety filters to simultaneously address safety and stabilization objectives.
Several CBF formulations have been developed for collision avoidance, and we focus only on those that explicitly account for the geometries of both bodies. Optimization-based approaches solve an optimization problem to compute a measure of separation,
such as a scaling factor \cite{Dai.etal.RAL23}, minimum distance \cite{Thirugnanam.etal.arXiv23,Fernandez.etal.RAL26}, smooth distance approximation \cite{Usevitch.Sahleen.ACC25}, or exact Signed Distance (SD) \cite{Chen.etal.arXiv26}. This measure is then used to define a CBF, and its gradient is obtained through differentiable optimization (diffOpt) \cite{Amos.Kolter.ICML17}. A similar workflow has also been applied to uncertain obstacles \cite{Matias.Silvestre.arXiv26} using constrained convex generators. In contrast, optimization-free approaches construct CBFs based on polytope vertices \cite{Molnar.CCTA25}, supporting hyperplanes for ellipsoids \cite{Funada.etal.TCST24} and polytopes \cite{Wu.etal.TCB25}, or Bernstein-polynomial SD fields \cite{Jo.etal.26}. Such approaches are computationally efficient, but generally rely on surrogate distance measures.

These two \fixed{classes} of methods appear to suggest a trade-off between geometric exactness and computational efficiency. 
However, although the SD sensitivity in our previous work \cite{Chen.etal.arXiv26} was derived through diffOpt, the resulting gradient expression was independent of the sensitivity of the optimizer but depended only on the optimizer itself, i.e., the witness pair (the pair of points realizing minimum distance when separated or penetration depth when intersecting).
Motivated by this, we propose \gjkcbf\footnote{Although the name suggests the use of GJK, the formulation only requires a witness pair, which may also be obtained via optimization.}, a CBF framework that retains the exact distance and obtains its gradient directly from the (instantaneous) witness pair without diffOpt (Fig.~\ref{fig: teaser}).
 

Related witness-pair-based distance gradients have appeared in earlier studies \cite{Choset.Lee.TRM01,Lee.Choset.IJRR05,Schulman.etal.IJRR14,Escande.etal.TRO14,Li.etal.IROS23}. Proximity queries like GJK have been incorporated into real-time CBF safety filters, while obtaining their gradients through diffOpt with randomized smoothing \cite{Zhao.etal.arXiv26} or a Schur-complement acceleration \cite{Guo.etal.IoT26}.
Another related work is \cite{Funada.etal.TCST24}, which designs CBFs using rotating supporting hyperplanes and derives a closed-form gradient by exploiting the special algebraic structure of ellipsoids, but does not directly generalize to other geometries.
The most closely related work to our approach is \cite{Wang.etal.CCDC25}, which formulates the distance-based CBF directly on $SE(3)$, but does not provide an explicit form for the distance gradient. 

\fixed{To the best of our knowledge, existing approaches do not provide a CBF formulation with an explicit expression for the exact distance gradient for general convex rigid bodies. To address this gap, we develop \gjkcbf, in which, once a witness pair is given, the exact distance is immediately determined, and its gradient can be obtained accordingly.}
The main contributions of this paper are as follows:
\begin{itemize}

    \item The proposed \gjkcbf~is geometry-agnostic within the class of convex rigid bodies and applies in both 2D and 3D workspaces, provided that at least one body in each collision pair is strictly convex \fixed{(Def.~\ref{def: strictly-conv})}. 
    \removeForSpace{\calin{I had this comment before: not sure what is a strictly convex set means. Strictly convex is usually used for functions} \yihsuan{I added the link to definition of strictly convex set.}}

    \item We derive an explicit expression for the distance gradient on $SE(3)$ solely from the witness pair without diffOpt, enabling optimization-based distance solvers to be replaced by GJK proximity queries. We also provide equivalent Euler-angle and quaternion representations.
    

    \item We show that the proposed CBF framework can be implemented in a decentralized fashion, assuming that the information associated with each collision pair is accessible to the corresponding robots.

    \item We validate our method across different 2D/3D scenarios and compare it 
    with the supporting-hyperplane-CBF 
    \cite{Funada.etal.TCST24}, \fixed{showing that our method is less conservative while remaining collision-free.}
\end{itemize}
The rest of this paper is organized as follows. Section~\ref{sec: prelim} provides preliminaries and Section~\ref{sec: prob} introduces the problem formulation. Section~\ref{sec: main} contains our main results on the proposed \gjkcbf. Section~\ref{sec: simu} presents simulation results and Section~\ref{sec: conc} concludes the paper.

\section{Preliminaries}\label{sec: prelim}
In this section, we provide background on CBFs (\ref{subsec: prelim_cbf}) and rigid-body kinematics (\ref{subsec: prelim_rbm}) to the extent needed for later development of our main results.

\subsection{Control Barrier Functions (CBFs)}\label{subsec: prelim_cbf}
Consider a nonlinear system in the control-affine form:
\begin{align}
    \dot{\bm{x}}=\mathbf{f}(\bm{x})+\mathbf{g}(\bm{x})\bm{u},\label{eq: ctrl-affine}
\end{align}
where $\bm{x}\in \mathbb{R}^n$ is the state, $\bm{u}\in\mathcal{U}\subseteq\mathbb{R}^m$ is the input, with $\mathcal{U}$ denoting a closed control constraint containing the origin,
and $\mathbf{f}:\mathbb{R}^n\rightarrow \mathbb{R}^n$, $\mathbf{g}:\mathbb{R}^n\rightarrow \mathbb{R}^{n\times m}$ are locally Lipschitz.

\begin{defi}[\it Forward invariant set \cite{Ames.etal.ECC19}]
    \normalfont A set $\mathcal{S}\subset\mathbb{R}^n$ is forward invariant for \eqref{eq: ctrl-affine} if, for every initial condition ${\bm{x}(0)\in \mathcal{S}}$, its unique solution satisfies $\bm{x}(t)\in \mathcal{S}$ for all time.
\end{defi}

Let $\mathcal{S}$ be the safe set, defined as the 0-superlevel set of a continuously differentiable function
$h:\mathbb{R}^n\rightarrow\mathbb{R}$
:
\begin{align}\begin{split} 
    \mathcal{S} &= \{\bm{x}\in \mathbb{R}^n:h(\bm{x})\geq 0 \}.
    \label{eq: safeset_def}
\end{split}\end{align}

\begin{defi}[\it CBF \cite{Ames.etal.ECC19}]
    \normalfont
    Given a set $\mathcal{S}$ as in \eqref{eq: safeset_def}, $h(\bm{x})$ is a control barrier function (CBF) for \eqref{eq: ctrl-affine} if there exists an extended class $\mathcal{K}$ function\footnote{%
    An extended class $\mathcal{K}$ function is a function $\alpha:\mathbb{R}\rightarrow\mathbb{R}$ that is strictly increasing and $\alpha(0)=0$, see, e.g., \cite{Ames.etal.ECC19}.
    } 
    $\alpha$ such that
    \begin{align}\begin{split}
    \textstyle\sup_{\bm{u}\in\mathcal{U}}  \dot{h}(\bm{x},\bm{u})&=\textstyle\sup_{\bm{u}\in\mathcal{U}}      [L_\mathbf{f}h(\bm{x})+L_\mathbf{g}h(\bm{x})\bm{u}]\\
    &\geq-\alpha(h(\bm{x})),\quad\forall\bm{x}\in \mathcal{S},
    \label{eq: cbf_constr}
    \end{split}\end{align}
    where 
    $L_\mathbf{f}(\cdot), L_\mathbf{g}(\cdot)$ denote the Lie derivatives along $\mathbf{f}$ and $\mathbf{g}$, respectively. 
    \label{def: cbf}
\end{defi}

\begin{thm}[\!\cite{Ames.etal.ECC19}]
    \normalfont Given a CBF $h$ with the associated safe set $\mathcal{S}$ from \eqref{eq: safeset_def}, any locally Lipschitz continuous controller $\bm{u}:\mathbb{R}^n\to\mathcal{U}$ that satisfies $\dot{h}(\bm{x},\bm{u})\geq-\alpha(h(\bm{x}))$, $\forall\bm{x}\in\mathcal {S}$
    renders the set $\mathcal{S}$ forward invariant for \eqref{eq: ctrl-affine}.
    \label{thm: safety-guarantee}
\end{thm}

Safe controllers can be synthesized by minimally modifying a given nominal controller, $\unom:\mathbb{R}^n\to\mathcal{U}$, subject to 
the CBF constraint derived from \eqref{eq: cbf_constr}.
For implementation, time $[0,T]$ is discretized into intervals of $\Delta t$. At each step, the state is treated as fixed at its value at the beginning of the interval, rendering the CBF constraints affine in $\bm{u}$. The control is applied over the interval and recomputed at the next step using the updated state. With the quadratic objective, this yields the following quadratic program (CBF-QP):
\begin{align}\begin{split}
    \bm{u}^*(\bm{x})&=\arg\min_{\bm{u}\in\mathcal{U}}\|\bm{u}-\unom(\bm{x})\|^2_{\mathbf{W}}
    \\ &\text{s.t.}\,\,\,\dot{h}(\bm{x},\bm{u})\geq-\alpha(h(\bm{x})),
    \label{eq: cbf-min-norm}
\end{split}\end{align}
where $\|\mathbf{x}\|_{\mathbf{W}}^2\triangleq\mathbf{x}^\top\mathbf{W}\mathbf{x}$, and $\mathbf{W}\succ0$ is a weighting matrix. 
\removeForSpace{\calin{it's not clear why the above is a QP. You need to say that time is discretized, state is kept constant at its value at the beginning of each interval, the constraints become affine, etc.}}

\subsection{Rigid Body Kinematics on Lie  Groups} \label{subsec: prelim_rbm}
Rigid-body motion consists of translation and rotation and can be naturally described by the Lie groups $SO(d)$ and $SE(d)$, with their Lie algebras $so(d)$ and $se(d)$, where $d\in\{2,3\}$ corresponds to 2D and 3D spaces.
To describe the configuration of a rigid body, we first introduce an inertial world frame $\mathcal{I}$ and a body-fixed frame $\mathcal{B}^i$ for rigid body $i$, whose origin can be arbitrarily chosen.
The configuration of rigid body $i$ with respect to $\mathcal{I}$ is described by ${g_i= (R_i,\bm{p}_i)=\begin{bmatrix} R_i & \bm{p}_i \\\bm{0} & 1 \\\end{bmatrix}\in SE(d)}
$, where ${R_i\in SO(d):=\{R\in\mathbb{R}^{d\times d}\mid RR^\top=\mathbf{I}_d,\det(R)=1\}}$ denotes its orientation and $\bm{p}_i\in\mathbb{R}^d$ its position (the origin of $\mathcal{B}^i$). The body velocity (twist) of rigid body $i$ with respect to $\mathcal{I}$, expressed in $\mathcal{B}^i$, is denoted by ${\bm{V}_i=[\bm{v}_i^\top, \bm{\omega}_i^\top]^\top}$
with ${\bm{v}_i \in\mathbb{R}^d}$ and ${\bm{\omega}_i\in\mathbb{R}^{(d(d-1))/2}}$ being the translational and angular velocities, respectively. The kinematics of rigid body $i$ are ${\dot{\bm{p}}_i = R_i\bm{v}_i,\,\,\,\dot{R}_i=R_i\widehat{\bm{\omega}}_i}$,
where 
the {\it hat} operator, $\widehat{(\cdot)}$, maps $\bm{\omega}_i$ from Euclidean space to Lie algebra $\widehat{\bm{\omega}}_i\in so(d)$, a vector space of all $d\times d$ skew-symmetric matrices. 

In 2D and 3D, the angular velocity is represented by a scalar $\omega\in\mathbb{R}$ and a vector $\bm{\omega}=[\omega_1,\omega_2,\omega_3]^\top\in\mathbb{R}^3$, respectively. The corresponding hat maps are
\begin{align}\label{eq: hat}
    \widehat\omega=\begin{bmatrix} 0 & -\omega \\ \omega & 0 \\ \end{bmatrix},\quad 
    \widehat{\bm{\omega}} = \begin{bmatrix}
                         0 & -\omega_3 & \omega_2 \\
                        \omega_3 & 0 & -\omega_1 \\
                        -\omega_2 & \omega_1 & 0 \\
                        \end{bmatrix}.
\end{align}
Equivalently, the kinematics can be written in homogeneous form as
\begin{align}
    \dot{g}_i = g_i \widehat {\bm{V}}_i,\quad\widehat {\bm{V}}_i:=\begin{bmatrix} \widehat{\bm{\omega}}_i & \bm{v}_i \\\bm{0} & 0 \\\end{bmatrix}\in se(d)
    \label{eq: rigid-body-kine-homo-rep}.
\end{align}
We refer interested readers to \cite{Li.etal.book94,solà.etal.arXiv21} for more details. 

\section{Problem Formulation}\label{sec: prob}

\removeForSpace{\calin{This whole section needs some rewriting. 
\begin{itemize} 
\item The distinction between the ``controlled" rigid bodies and obstacles is not clear. I would call the controlled rigid bodies robots, or robots, which you do sometimes but you're not consistent. Just say that we have a set of robots (robots) modeled as rigid bodies with dynamics given by (7) and a set of obstacles that are static rigid bodies. And you should have this at the beginning of the section. \yihsuan{i replaced all the agents with robots.}
\item are there any bounds on the controls? This is usually central to CBF-based methods. \yihsuan{i did not consider any control bounds to avoid infeasibility issues.}
\item The assumption of the strict convexity is a bit complicated and confusing. Can we say that the robots are strictly convex? I think this actually might cover the assumption that you need \yihsuan{I rewrote the assumptions and hopefully it is clear now}
\item where does the nominal control come from? you might want to say something about this \yihsuan{I added this in the assumptions too}
\end{itemize}
}
\yihsuan{I rewrote the following to make the assumptions and problem formulation clearer.}
}

This work considers collision avoidance among ${N\geq1}$ robots and ${M\geq0}$ obstacles, each modeled as a convex rigid body. Their respective index sets are
${\mathcal{N}:=\{1,\ldots,N\}}$ and ${\mathcal{M}:=\{1,\ldots,M\}}$, with ${\mathcal{M}=\emptyset}$ for $M=0$. The workspace ${\mathcal{W}\subset\mathbb{R}^d}$, ${d\in\{2,3\}}$, is the physical (2D or 3D Euclidean) environment in which the robots move and avoid obstacles. The regions occupied by robot $i$ and obstacle $o$ are denoted by $\mathcal{A}_i(g_i)\subset\mathcal{W}$ and $\mathcal{O}_o\subset\mathcal{W}$, respectively.
Unless otherwise specified, we use ${\bm{x}_i=g_i\in SE(d)}$ as the state and $\bm{u}_i=\bm{V}_i$ as the input for robot $i\in\mathcal{N}$.
Alternative Euler-angle and quaternion state representations for 3D
are 
${\bm{x}_i^E\in\mathbb{R}^6}$ and ${\bm{x}_i^Q\in\mathbb{R}^7}$, respectively. Before stating the assumptions, we recall the notions of convex and strictly convex sets.

\begin{defi}[\it Convex and Strictly Convex Sets \cite{Escande.etal.TRO14}]
\label{def: strictly-conv}
    A set $\mathcal{C}$ is {\it convex} if the point $\mathbf{p}=\lambda \mathbf{x}+ (1-\lambda)\mathbf{y}$ 
    lies in $\mathcal{C}$ for any $\mathbf{x},\mathbf{y}\in\mathcal{C}$ and any $\lambda\in[0,1]$, and {\it strictly convex} if, for any $\mathbf{x}\neq\mathbf{y}$ and any $\lambda\in(0,1)$, $\mathbf{p}$ lies in  the interior of $\mathcal{C}$.
\end{defi}

Our problem is defined under the following assumptions:
\begin{enumerate}[label=(\roman*)]
    \item The occupied sets $\mathcal{A}_i(\bm{x}_i)$ of each robot $i\in\mathcal{N}$ and $\mathcal{O}_o$ of each obstacle $o\in\mathcal{M}$, are compact and convex with nonempty interiors; all obstacles are static.
    \item For every collision pair, whether robot-robot or robot-obstacle, at least one of the two sets is strictly convex.\label{assump: conv}
    \item Each robot has access to the current states of all other robots it must avoid, as well as information about the obstacles; inter-robot communication is perfect, i.e., no delays or packet losses.\label{assump: commu}

    \item Each robot is provided with a nominal go-to-goal controller $\unom_i,~\forall i\in \mathcal{N}$, which is not necessarily safe.
\end{enumerate}
The collision avoidance problem considered in this work is formulated as follows:
\begin{prob}
    \normalfont
    Consider a workspace $\mathcal{W}$ containing ${M\geq0}$ obstacles and ${N\geq1}$ robots. Each robot ${i\in\mathcal{N}}$ evolves according to \eqref{eq: rigid-body-kine-homo-rep} with a convex polyhedral admissible control set ${\mathcal{U}_i}$. 
    Given collision-free initial and goal configurations, ${\{\bm{x}^0_i\}}$ and ${\{\bm{x}^d_i\}}$,
    and nominal controllers ${\{\unom_i\}}$, design safe controllers ${\{\bm{u}_i\in\mathcal{U}_i\}}$ that minimally modify ${\{\unom_i\}}$ while avoiding collisions with one another and all obstacles.
\end{prob}


\section{Main Results}\label{sec: main}
This section presents the proposed \gjkcbf. We first establish its differentiability and derive its explicit gradient on $SE(d)$ (\ref{subsec: min-dist}) and then give equivalent forms in the commonly used Euler angles and quaternions
(\ref{subsec: alt-hdot}). Finally, we present a \fixed{decentralized} implementation, in the sense that each robot independently solves for its own control (\ref{subsec: decentralized}).
\removeForSpace{\calin{not clear what distributed means. In the previous section you said that each robot has access to all information.}\yihsuan{I changed the wording from distributed to decentralized.}}

\subsection{Minimum Distance as a CBF}
\label{subsec: min-dist}
Collision avoidance between two rigid bodies can be characterized by maintaining a positive (minimum) distance between their occupied sets. \fixed{We first consider a pair of robots ${i,j\in\mathcal{N},i\neq j}$.}
\removeForSpace{\calin{not sure why you only consider robots. This should apply to robot-obstacle pairs, right? \yihsuan{yes, i added one sentence after (8-9) to explain this.}}}
Let ${\mathcal{G}_{ij} := (g_i,g_j)}$. The pair of closest points ${(\bm{a}_i^*(\mathcal{G}_{ij}),\bm{a}_j^*(\mathcal{G}_{ij}))}$, expressed in $\mathcal{I}$, is referred to as the witness pair ${\mathbf{w}(\mathcal{G}_{ij})}$, defined 
alongside the distance as
\begin{align}
&\mathbf{w}(\mathcal{G}_{ij}):=(\bm{a}_i^*,\bm{a}_j^*)=\arg\min_{\substack{
\bm{a}_i\in\mathcal{A}_i(g_i)\\
\bm{a}_j\in\mathcal{A}_j(g_j)
}}\|\bm{a}_i-\bm{a}_j\|^2_2,\label{eq: wp}
\\
&\mathrm{dist}(\mathbf{w}(\mathcal{G}_{ij}))=\|\bm{a}^*_i-\bm{a}^*_j\|_2.\label{eq: dist}
\end{align}
\fixed{The robot-obstacle case follows analogously from \eqref{eq: wp}-\eqref{eq: dist} by replacing $\mathcal{A}_j(g_j)$ with $\mathcal{O}_o$ and omitting the dependence on $g_j$.}
We then propose the following CBF candidate:
\begin{align}\begin{split} 
    h(\mathcal{G}_{ij})=\|\bm{a}^*_i-\bm{a}^*_j\|_2-d_\mathrm{safe},
    \label{eq: h_dist}
\end{split}\end{align}
where $d_\mathrm{safe}>0$ is a user-defined safety margin. \fixed{For brevity, we denote dependence on $\mathcal{G}_{ij}$
by the subscript ${ij}$, i.e., ${(\cdot)_{ij}=(\cdot)(\mathcal{G}_{ij})}$, when clear from context.} The associated safe set is given by $\mathcal{S}_{ij} =\{\mathcal{G}_{ij}\in SE(d)\times SE(d):h_{ij}\geq 0 \}$.  

Although the witness pair $\mathbf{w}_{ij}$ is defined through the optimization in \eqref{eq: wp}, our formulation requires only $\mathbf{w}_{ij}$ itself to compute the gradient of \eqref{eq: h_dist} (see Theorem~\ref{thm: hdot}), without dual variables as in diffOpt-based methods \cite{Thirugnanam.etal.arXiv23,Fernandez.etal.RAL26,Usevitch.Sahleen.ACC25}.
Thus, $\mathbf{w}_{ij}$ can instead be obtained using proximity-query algorithms such as GJK \cite{gjk}, which applies to general convex geometries. Alternatively, for particular set representations, \eqref{eq: wp} can be formulated as a convex program, e.g., a QP for polytopes, a quadratically constrained QP (QCQP) for ellipsoids, or a general convex nonlinear program for sets defined by nonlinear convex constraints.
It is worth noting that computing $\mathbf{w}_{ij}$ does not depend on Assumption~\ref{assump: conv}, which is required only to ensure the differentiability of the proposed CBF \eqref{eq: h_dist}, as established next.
\begin{lemma}\label{lem: c1}
    Under Assumption~\ref{assump: conv}, for ${\mathcal{G}_{ij}\in\mathcal{S}_{ij}}$,
    the witness pair $\mathbf{w}_{ij}$ \eqref{eq: wp} is unique and continuous, and the CBF $h_{ij}$ in \eqref{eq: h_dist} is continuously differentiable ($C^1$) with respect to $\mathcal{G}_{ij}$.
\end{lemma}
\begin{proof}
    The uniqueness and continuity of $\mathbf{w}_{ij}$ in \eqref{eq: wp} and the $C^1$ regularity of the CBF in \eqref{eq: h_dist} with respect to the relative configuration $g_{\mathrm{rel}} = g_i^{-1} g_j$ follow directly from \cite[Sec.~IV-B]{Escande.etal.TRO14}. Since $g_{\mathrm{rel}}$ and the transformation to $\mathcal{I}$ depend smoothly on $\mathcal{G}_{ij}$, these properties also hold with respect to $\mathcal{G}_{ij}$.
\end{proof}

We also remark that the standard CBF \cite{Ames.etal.ECC19} is defined with respect to systems evolving in Euclidean spaces in the form of \eqref{eq: ctrl-affine}. The proposed CBF \eqref{eq: h_dist}, however, is defined on a manifold of $SE(d)$. Recent works have extended CBF theory to manifolds and Lie groups \cite{De.etal.ACC26,Wang.etal.CCDC25,Lin.Egerstedt.arXiv26}. Forward invariance of the safe set can similarly be established by 
requiring the time derivative of $h_{ij}$ 
along the system trajectories to satisfy $\dot{h}_{ij}\geq-\alpha(h_{ij})$
\cite{Funada.etal.TCST24,De.etal.ACC26}. Before giving the result for $\dot{h}_{ij}$, we introduce the operator $(\cdot)^\perp$ for ${\bm{a}=[a_1,\ldots,a_d]^\top\in\mathbb{R}^d}$:
\begin{align}
    \bm{a}^\perp
    := 
    \begin{cases}
        [-a_2, a_1]^\top, & d=2\\
        -\widehat{\bm{a}}, & d=3
    \end{cases}.
\end{align}
Here,   ${\bm{a}^\perp\in\mathbb{R}^{d\times d(d-1)/2}}$ satisfies ${\widehat{\bm{\omega}}\bm{a}=\bm{a}^\perp\bm{\omega}}$, with the hat operator defined in \eqref{eq: hat}. We now state the formula for $\dot{h}_{ij}$.
\begin{thm}\label{thm: hdot}
    \normalfont 
    Under Assumption~\ref{assump: conv}, for $\mathcal{G}_{ij}\in\mathcal{S}_{ij}$, the time derivative of the CBF \eqref{eq: h_dist} can be computed solely from the witness pair $\mathbf{w}_{ij}=(\bm{a}_i^*,\bm{a}_j^*)$, the current configurations $\mathcal{G}_{ij}$, and the body-velocity controls $\bm{u}_i$ and $\bm{u}_j$
    as 
    \begin{subequations}\begin{align}
        &\dot{h}_{ij}= \mathbf{J}_i \bm{u}_i +\mathbf{J}_j\bm{u}_j,
        \label{eq: hdot}
        \\[0.35em]
        &
        \mathbf{n}_{ij}=\frac{\bm{a}_i^*-\bm{a}_j^*}{\|\bm{a}_i^*-\bm{a}_j^*\|_2},\,\,\,\mathbf{n}_{ji}=-\mathbf{n}_{ij}\label{eq: nij}
        \\
        &\scalebox{0.88}{$\displaystyle(\mathbf{J}_i)^\top=
        \begin{bmatrix}
            R_i^\top\mathbf{n}_{ij}\\[0.2em] (R_i({\bm{\alpha}^*_i)^\perp})^\top\mathbf{n}_{ij}
        \end{bmatrix},\,\,
        (\mathbf{J}_j)^\top=
        \begin{bmatrix}
            R_j^\top\mathbf{n}_{ji}\\[0.2em] (R_j({\bm{\alpha}^*_j)^\perp})^\top\mathbf{n}_{ji}
        \end{bmatrix},$}
        \label{eq: Jaco}\\[0.35em]
        &\bm{\alpha}_k^*=R_k^\top(\bm{a}_k^*-\bm{p}_k),\,\,k\in\{i,j\}.\label{eq: alpha}
    \end{align}\label{eq: hdot_se3}\end{subequations}
    Here, ${\bm{p}_k\in\mathbb{R}^d}$ is the position of the origin of ${\mathcal{B}^k}$, expressed in $\mathcal{I}$, and
    ${\bm{\alpha}_k^*\in\mathbb{R}^d}$ is the vector from $\bm{p}_k$ to $\bm{a}_k^*$, expressed in $\mathcal{B}^k$. We refer to ${\mathbf{J}_k\in\mathbb{R}^{1\times(d+d(d-1)/2)}}$ 
    as the \emph{body-velocity Jacobian} associated with $h_{ij}$. Finally, $\mathbf{n}_{ij}\in\mathbb{R}^d$ is the unit vector from $\bm{a}_j^*$ to $\bm{a}_i^*$, and vice versa for $\mathbf{n}_{ji}$.
\end{thm}
\begin{proof} 
    Let $\mathcal{G}:=\mathcal{G}_{ij}$, $\overline{\mathcal{A}}_k$ denote the representation of rigid body $k$ in $\mathcal{B}^k$, and let $\bm{z}:=(\bm{\alpha}_i,\bm{\alpha}_j) \in \overline{\mathcal {A}}_i\times\overline{\mathcal{A}}_j$. Define $\Phi(\mathcal{G},\bm{z}):=\|\bm{p}_i+R_i\bm{\alpha}_i-\bm{p}_j-R_j\bm{\alpha}_j\|_2$. The minimum distance between two bodies is then $\mathrm{dist}(\mathcal{G})=\min_{\bm{z}}\Phi(\mathcal{G},\bm{z})$, with $\bm{z}^*(\mathcal{G}):=(\bm{\alpha}_i^*(\mathcal{G}),\bm{\alpha}_j^*(\mathcal{G}))$ denoting its unique minimizer by Lemma~\ref{lem: c1}. Along $\mathcal{G}(t)$, let $\bm{z}^*_t:=\bm{z}^*(\mathcal{G}(t))$. For any sufficiently small $\Delta>0$, optimality of $\bm{z}_{t+\Delta}^*$ at $\mathcal{G}(t+\Delta)$ gives $\Phi(\mathcal{G}(t+\Delta),\bm{z}_{t+\Delta}^*)\leq\Phi(\mathcal{G}(t+\Delta),\bm{z}_{t}^*)$, whereas optimality of $\bm{z}^*_t$ at $\mathcal{G}(t)$ gives $\Phi(\mathcal{G}(t),\bm{z}_{t}^*)\leq\Phi(\mathcal{G}(t),\bm{z}_{t+\Delta}^*)$. Since $\mathrm{dist}(\mathcal{G}(\tau))=\Phi(\mathcal{G}(\tau),\bm{z}^*_\tau)$, for every $\tau$, combining the two inequalities above yields
    \begin{align*}
        \Phi&(\mathcal{G}(t+\Delta),\bm{z}_{t+\Delta}^*)-\Phi(\mathcal{G}(t),\bm{z}_{t+\Delta}^*)\\
        &\leq\mathrm{dist}(\mathcal{G}(t+\Delta))-\mathrm{dist}(\mathcal{G}(t))\\
        &\leq \Phi(\mathcal{G}(t+\Delta),\bm{z}_t^*)-\Phi(\mathcal{G}(t),\bm{z}_t^*),
    \end{align*}
    By Lemma~\ref{lem: c1}, the unique minimizer is continuous, so that $\bm{z}_{t+\Delta}^*\to\bm{z}_t^*$ as $\Delta\to0$. Since $\Phi$ is $C^1$ with respect to $\mathcal{G}$ whenever $\mathrm{dist}(\mathcal{G})>0$, division by $\Delta$ (reversing the inequalities if $\Delta<0$), followed by the mean value theorem and the squeeze theorem, yields
    \begin{align*}
        \frac{d}{dt}\mathrm{dist}(\mathcal{G}(t))=\frac{\partial}{\partial t} \Phi(\mathcal{G}(t),\bm{z})\Big\vert_{\bm{z}=\bm{z}_t^*}.
    \end{align*}
    Since $d_\mathrm{safe}$ is constant, we evaluate the partial derivative above while holding fixed the instantaneous witness points $\bm{\alpha}^*_k$ defined in \eqref{eq: alpha}. Using \eqref{eq: nij}, we obtain
    \begin{align*}
        \dot{h}_{ij}=\frac{d}{dt}\mathrm{dist}(\mathcal{G}(t)
        )=\mathbf{n}_{ij}^
        \top\left( \dot{\bm{p}}_i+\dot{R}_i\bm{\alpha}^*_i-\dot{\bm{p}}_j-\dot{R}_j\bm{\alpha}^*_j
        \right).
    \end{align*}
    Substituting ${\dot{\bm{p}}_k=R_k\bm{v}_k}$ and ${\dot{R}_k=R_k\widehat{\bm{\omega}}_k}$ into the above equation and using ${\bm{n}^\top R\widehat{\bm{\omega}}\bm{\alpha}=\bm{n}^\top R\bm\alpha^\perp\bm\omega}$ and $\mathbf{n}_{ji}=-\mathbf{n}_{ij}$ yields $\dot{h}_{ij}=\mathbf{J}_i\bm{u}_i+\mathbf{J}_j\bm{u}_j$, with $\mathbf{J}_i$ and $\mathbf{J}_j$ given by \eqref{eq: Jaco}. This completes the proof.
\end{proof}

Since the rigid-body kinematics \eqref{eq: rigid-body-kine-homo-rep} are driftless and the zero input is admissible, choosing $\bm{u}_i=\bm{u}_j=\bm{0}$ yields $\dot{h}_{ij}=0\geq-\alpha(h_{ij})$ for every $\mathcal{G}_{ij}\in\mathcal{S}_{ij}$; hence, $h_{ij}$ satisfies the CBF condition \eqref{eq: cbf_constr}. By Theorem~\ref{thm: safety-guarantee}, any locally Lipschitz controller satisfying \eqref{eq: cbf_constr} renders $\mathcal{S}_{ij}$ forward invariant, thereby maintaining separation for all time.

\begin{remark}
\label{rem: obs}
For a robot-obstacle collision pair $(i,o)$, $i\in\mathcal{N}$, $o\in\mathcal{M}$, we denote its witness pair, CBF, safe set, and separation normal by $\mathbf w_{io}$, $h_{io}$, $\mathcal{S}_{io}$, and $\mathbf{n}_{io}$, respectively. Since the obstacle is static, its velocity term vanishes, and Theorem 2 reduces to $\dot h_{io}=\mathbf{J}_i\bm{u}_i$.
\end{remark}

In our previous work \cite{Chen.etal.arXiv26}, the 2D counterpart of Theorem~\ref{thm: hdot} was derived using diffOpt together with the Minkowski-operation property, revealing that the gradient depends only on the witness pair. This observation motivates the present 3D derivation based on rigid-body motion.
The same derivative expression also applies to the SD formulation in the intersecting case, provided the penetration witness pair is unique. Nonunique witness pairs introduce nonsmoothness, which was analyzed for 2D polygons in the same work; here, we focus on the $C^1$ setting established in Lemma~\ref{lem: c1}.

We note that for force- or torque-controlled systems, distance-based CBFs such as our \gjkcbf~have relative degree two, i.e., the control first appears in the second time derivative of the CBF. A direct extension requires a higher-order CBF \cite{Xiao.Belta.TAC22}, which involves the distance Hessian and thus sensitivity analysis of the optimal witness points $\bm a_k^*$ \cite{Montaut.etal.ICRA23,Wang.etal.CCDC25,Chen.etal.CDC26,Zimmermann.etal.IROS22}. Alternatively, specially designed CBFs, such as those proposed in \cite{Singletary.etal.LCSS22}, combine kinetic energy with the distance-based safety function to avoid explicit computation of the distance Hessian.




\subsection{Alternative Attitude Representations in 3D}\label{subsec: alt-hdot}
In 3D, Euler angles provide an intuitive and minimal parameterization but suffer from singularities, whereas quaternions avoid them at the cost of an additional parameter. We now present equivalent expressions of the \gjkcbf~derivative using Euler angles and quaternions. While using different state representations, we retain the body velocity ${\bm{u}_i=[\bm{v}_i^\top, \bm{\omega}_i^\top]^\top}$
as the control input.


\textbf{Euler angles.}
The state is
$\bm{x}_i^E=[\bm{p}_i^\top,\bm{\eta}_i^\top]^\top$, where $\bm{\eta}_i=[\phi_i,\theta_i,\psi_i]^\top\in\mathbb{R}^3$ denotes the roll, pitch, yaw angles, respectively. We adopt the common yaw-pitch-roll convention, i.e., 
$R_i=R_z(\psi_i)R_y(\theta_i)R_x(\phi_i)$,
with
\vspace{-0.5em}
\par\begingroup\small
\begin{align*}
        R_z = \begin{bmatrix}
              c & -s & 0\\
              s & c & 0 \\
                0 &    0 & 1
        \end{bmatrix},
        R_y = \begin{bmatrix}
              c & 0 & s\\
              0 & 1 & 0 \\
              -s & 0 & c
            \end{bmatrix},
        R_x = \begin{bmatrix}
              1 & 0 & 0\\
              0 & c & -s \\
              0 & s & c
            \end{bmatrix},
\end{align*}
\endgroup
where $c$ and $s$ denote the cosine and sine of each matrix's angle, respectively. The kinematics are given by:
\begin{subequations}\label{subeq: euler-rep}
    \begin{align}
        \dot{\bm{x}}_i^E &=\begin{bmatrix}
        R_i & \bm{0} \\
        \bm{0} & {T}(\bm{\eta}_i)^{-1}
        \end{bmatrix}\bm{u}_i\triangleq\mathbf{E}_i\bm{u}_i,\label{eq: euler-kine}\\[0.3em]
        {T}(\bm{\eta}_i) &= 
        \begin{bmatrix}
            1 & 0 & -\sin(\theta_i) \\
            0 & \cos(\phi_i)& \sin(\phi_i)\cos(\theta_i)\\ 
            0 &-\sin(\phi_i) & \cos(\phi_i)\cos(\theta_i)
        \end{bmatrix},\label{eq: euler-trans}
    \end{align}
\end{subequations}
Using \eqref{eq: euler-kine}, the time derivative of $h_{ij}$ in Euler-angle coordinates can be written as
\begin{align}
    \dot{h}_{ij}= \mathbf{J}_i^E\dot{\bm{x}}_i^E+\mathbf{J}_j^E\dot{\bm{x}}_j^E=\mathbf{J}_i^E\mathbf{E}_i\bm{u}_i+\mathbf{J}_j^E\mathbf{E}_j\bm{u}_j,
    \label{eq: hdot-euler}
\end{align}
where $\mathbf{J}_k^E=(\nabla_{\bm{x}_k^E}h_{ij})^\top$ denotes the Jacobian of $h_{ij}$ with respect to the Euler-angle coordinates of body $k$, for $k\in\{i,j\}$. By Theorem~\ref{thm: hdot}, it is given by
\begin{align}
    \mathbf{J}_k^E = \mathbf{J}_k\left(\mathbf{E}_k\right)^{-1}= \mathbf{J}_k\begin{bmatrix}
        R_k^\top & \bm{0} \\
        \bm{0} & {T}(\bm{\eta}_k)
        \end{bmatrix}.
    \label{eq: jaco-euler}
\end{align}
This expression is valid away from the Euler-angle singularity, i.e., for $\cos\theta_k\neq0$. A related result was previously obtained in~\cite[Appendix~A]{Choset.Lee.TRM01} for rod-shaped robots via a different derivation. Our result in \eqref{eq: jaco-euler} reduces to theirs when the roll angle is set to zero.

\textbf{Quaternion.} The same approach applies to the unit quaternion representation, with state ${\bm{x}_i^Q=[\bm{p}_i^\top,\bm{q}_i^\top]^\top}$, where ${\bm{q}_i=[q_{0_i}, \bm{q}_{v_i}^\top]^\top \in\mathbb{R}^4}$ and ${\|\bm{q}_i\|_2=1}$. The corresponding kinematics are:
\begin{subequations}\label{subeq: quat-rep}
    \begin{align}
        \dot{\bm{x}}_i^Q &=\begin{bmatrix}
        R_i & \bm{0} \\
        \bm{0} & \frac{1}{2}{G}(\bm{q_i})\end{bmatrix}\bm{u}_i\triangleq\mathbf{Q}_i\bm{u}_i,\label{eq: quat-kine}
        \\[0.3em]
        {G}(\bm{q}_i) &= 
        \begin{bmatrix}
        -\bm{q}_{v_i}^\top \\ q_{0_i}\mathbf{I}_3+\widehat{\bm{q}_{v_i}}
        \end{bmatrix},\label{eq: quat-trans}
    \end{align}
\end{subequations}
Similarly, using \eqref{eq: quat-kine}, the time derivative of $h_{ij}$ in quaternion coordinates can be written as
\begin{align}
    \dot{h}_{ij}= \mathbf{J}_i^Q\dot{\bm{x}}_i^Q+\mathbf{J}_j^Q\dot{\bm{x}}_j^Q=\mathbf{J}_i^Q\mathbf{Q}_i\bm{u}_i+\mathbf{J}_j^Q\mathbf{Q}_j\bm{u}_j,
    \label{eq: hdot-quat}
\end{align}
where $\mathbf{J}_k^Q$ denotes one valid quaternion-coordinate Jacobian associated with body $k$. Since $G(\bm{q}_k)^\top G(\bm{q}_k)=\mathbf{I}_3$ holds for unit quaternions, a left inverse of $\mathbf{Q}_k$ yields the following valid Jacobian via Theorem~\ref{thm: hdot}:
\begin{align}
    \mathbf{J}_k^Q=\mathbf{J}_k
    \begin{bmatrix}
    R_k^\top & \bm{0} \\
    \bm{0} &  2{G}(\bm{q_k})^\top\\ 
    \end{bmatrix}.
    \label{eq: jaco-quat}
\end{align}

In summary, depending on the chosen attitude representation, the CBF constraint can be equivalently imposed by evaluating $\dot{h}_{ij}$ using \eqref{eq: hdot_se3} for $SE(d)$ representation, or, for $d=3$, using \eqref{eq: hdot-euler}-\eqref{eq: jaco-euler} for Euler angles, or \eqref{eq: hdot-quat}-\eqref{eq: jaco-quat} for unit quaternions. In the subsequent decentralized CBF formulation and simulations, we adopt the default $SE(d)$ expression \eqref{eq: hdot_se3}.

\subsection{Decentralized Implementation}\label{subsec: decentralized}
With the CBF derivative available, we now synthesize safe control inputs by solving the safety-filter CBF-QP \eqref{eq: cbf-min-norm}. Under Assumption~\ref{assump: commu}, following \cite{Wang.etal.TRO17}, each pairwise CBF constraint can be split between the corresponding robots without compromising safety, yielding a decentralized implementation in which each robot solves for its own control input. To describe this allocation, consider a fixed robot pair $(i,j)$. Each robot independently computes the CBF value $h_{ij}$ from \eqref{eq: h_dist}
and its own body-velocity Jacobian $\mathbf{J}_k, k\in\{i,j\}$ using \eqref{eq: Jaco}. For this pair, their local CBF-QPs are:
\begin{subequations}
\begin{align}
    \bm{u}_i^*&=\arg\min_{\bm{u}_i\in\mathcal{U}_i}\|\bm{u}_i-\unom_{i}\|^2_{\mathbf{W}_i} \notag\\
    &\quad\text{s.t.}\,\,\,\mathbf{J}_i \bm{u}_i\geq-\beta_{ij}\alpha(h_{ij}), \label{eq: i-cbf-qp}\\[0.3em]
    \bm{u}_j^*&=\arg\min_{\bm{u}_j\in\mathcal{U}_j}\|\bm{u}_j-\unom_{j}\|^2_{\mathbf{W}_j} \notag\\
    &\quad\text{s.t.}\,\,\,\mathbf{J}_j \bm{u}_j\geq-(1-\beta_{ij})\alpha(h_{ij}), \label{eq: j-cbf-qp}
\end{align}\label{eq: decentralized-cbf-qp}\end{subequations}
where ${0\leq\beta_{ij}\leq1}$ is the user-defined parameter that allocates the avoidance responsibility within the pair and may be chosen according to their control capabilities. 

Summing the two constraints in
\eqref{eq: i-cbf-qp}-\eqref{eq: j-cbf-qp} and using \eqref{eq: hdot} yields ${\mathbf{J}_i\bm{u}_i+\mathbf{J}_j\bm{u}_j=\dot{h}_{ij}\geq-\alpha(h_{ij})}$. That is, under perfect communication in Assumption~\ref{assump: commu}, the local constraints \eqref{eq: i-cbf-qp}-\eqref{eq: j-cbf-qp} recover the original pairwise CBF condition.
Then, provided that the local QPs are feasible, the safe set $\mathcal{S}_{ij}$ is forward invariant by Theorem~\ref{thm: safety-guarantee}. For $N$ robots, there are $\binom{N}{2}$ distinct inter-robot CBF constraints. Each robot enforces its at most ${N-1}$ associated constraints in a single local QP.
With $M$ obstacles, the same QP additionally includes $M$ robot-obstacle CBF constraints. Algorithm~\ref{alg: decen-gjk-cbf} summarizes the decentralized implementation, with each robot independently performing its own computations.
\makeatletter
\newcommand{\StateNoNumber}[1]{%
  \begingroup
  \algrenewcommand\alglinenumber[1]{}%
  \State #1%
  \addtocounter{ALG@line}{-1}%
  \endgroup
}
\makeatother
\begin{algorithm}[htb!]
\caption{Decentralized GJK-CBF Safety Filter}
\label{alg: decen-gjk-cbf}
\begin{algorithmic}[1]
\Require Robots $\mathcal{N}$ with
$\{\bm{x}_i^0\}$,
$\{\bm{x}_i^d\}$,
$\{\unom_i\}$,
$\{\mathbf{W}_i\}$, $\forall i\in\mathcal{N}$,
$\bm{\beta}
=\{\beta_{ij}\in[0,1]\}_{1\leq i<j\leq N}$, static obstacles $\{\mathcal{O}_o\}_{o\in\mathcal{M}}$,
$d_{\mathrm{safe}}$,
$\gamma$, $T$, and $\Delta t$
\Ensure Safe control and state trajectories for all robots
\State $K\gets T/\Delta t$ 
\For{$k=0,\ldots,K-1$}
    \ForAll{$i\in\mathcal N$} 
        \State $A_i\gets[\,]$, $b_i\gets[\,]$ \Comment{Init. CBF-QP constraints}
        \State $\unom_i^k\gets\unom_i(\bm{x}_i^k,\bm{x}_i^d)$
        \Comment{Nominal control 
        }
    \EndFor
    \ForAll{$i,j\in\mathcal{N}$, $i<j$
    } \Comment{Inter-robot safety}
        \State $(\bm{a}_i^*,\bm{a}_j^*,d_{ij})
        \gets\texttt{GJK}
        \bigl(\mathcal{A}_i(\bm{x}_i^k),
        \mathcal{A}_j(\bm{x}_j^k)\bigr)$
        \State $h_{ij}\gets d_{ij}-d_{\mathrm{safe}}$ using \eqref{eq: h_dist} \Comment{CBF value}
        \State Compute $\mathbf{J}_i$ and $\mathbf{J}_j$ using \eqref{eq: nij}-\eqref{eq: alpha}
        \State $A_i\gets[A_i;-\mathbf{J}_i]$,
        \quad
        $b_i\gets[b_i;\beta_{ij}\gamma h_{ij}]$
        \State $A_j\gets[A_j;-\mathbf{J}_j]$,
        \quad
        $b_j\gets[b_j;(1-\beta_{ij})\gamma h_{ij}]$
    \EndFor
        \ForAll{$(i,o)\in\mathcal{N}\times\mathcal{M}
        $}\Comment{robot-obstacle safety}
        \State $(\bm{a}_i^*,\bm{a}_o^*,d_{io})\gets\texttt{GJK}
            \bigl(\mathcal{A}_i(\bm{x}_i^k),
            \mathcal{O}_o\bigr)$
        \State $h_{io}\gets d_{io}-d_{\mathrm{safe}}$ using \eqref{eq: h_dist} \Comment{CBF value}
        \State Compute $\mathbf{J}_i$ using \eqref{eq: nij}-\eqref{eq: alpha}
        \State $A_i\leftarrow[A_i;-\mathbf{J}_i],\quad b_i\leftarrow[b_i;\gamma h_{io}]$
        \EndFor
    \ForAll{$i\in\mathcal N$} \Comment{Solve CBF-QP \eqref{eq: cbf-min-norm}}
        \State $\bm{u}_{i}^k\gets
        {\min}_{\bm{u}_i\in\mathcal{U}_i}
        \|\bm{u}_i-\unom_i^k\|_{\mathbf{W}_i}^2
        \,\,\,\mathrm{s.t.}\,\,\,A_i \bm{u}_i\le b_i$
        \State $\bm{x}_i^{k+1}\gets
        \texttt{Propagate}({\bm{x}_i^k,\bm{u}_i^k,\Delta t}$) 
        \StateNoNumber using \eqref{eq: rigid-body-kine-homo-rep},
        \eqref{eq: euler-kine}, or \eqref{eq: quat-kine}
    \EndFor
\EndFor
\end{algorithmic}
\end{algorithm}
\begin{figure*}[t]
    \makebox[\linewidth][c]{
        \begin{xy}
            \xyimport(100,100){\includegraphics[width=0.964\textwidth]{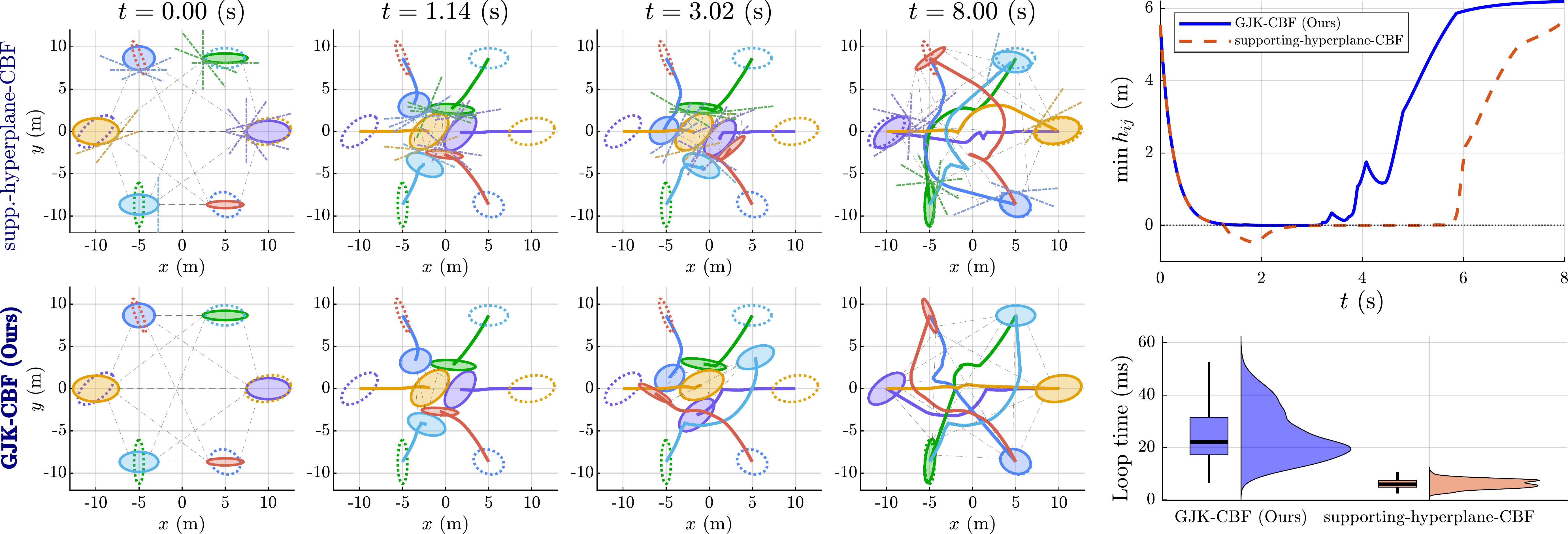}}
            ,(0.5,95.2)*+{\rotatebox{90}{\scriptsize\cite{Funada.etal.TCST24}}}
        \end{xy}
    }
    \caption{\textbf{Left:} Comparison between \textbf{(top)} the supporting-hyperplane-CBF \cite{Funada.etal.TCST24} and \textbf{(bottom)} the proposed \gjkcbf~for six ellipsoidal robots swapping positions. Dotted outlines denote the desired poses, and solid curves show the robot trajectories. Dash-dotted lines in the top row represent the supporting hyperplanes. \textbf{Right:} \textbf{(top)} minimum pairwise CBF values and \textbf{(bottom)} control-loop computation-time distributions, 
    which are measured for the serial implementation and include all pairwise CBF evaluations and local QP solves for the six robots.}
    \label{fig: cmp}
\end{figure*}

\section{Simulation Results}\label{sec: simu}
We validate the proposed \gjkcbf~through 2D and 3D simulations, including comparisons with \cite{Funada.etal.TCST24} in 2D multi-robot position swapping and 3D slit passage (\ref{subsec: simu-2d}), followed by 3D position swapping (\ref{subsec: simu-3d}), and finally demonstrate its applicability to a velocity-controlled manipulator (\ref{subsec: simu-arm}).

Following \cite{Funada.etal.TCST24}, in all simulations except the manipulator example, we use the following nominal controllers ${\vnom_i=k_{v,i}R_i^\top(\bm{p}_i^d-\bm{p}_i)}$, ${\wnom_i= \frac{k_{\omega,i}}{2}(R_i^\top R_i^d-(R_i^\top R_i^d)^\top)^\vee}$\footnote{The vee map $(\cdot)^\vee:so(d)\to\mathbb{R}^{d(d-1)/2}$ is the inverse of the hat map defined in \eqref{eq: hat}.}.
All simulations are conducted in MATLAB, with QPs solved using \texttt{quadprog}. We set $\alpha(h)=\gamma h$ with $\gamma=4$, $\beta_{ij}=0.5$, and $d_{\mathrm{safe}}=0.05$~m unless otherwise specified. Control bounds are imposed only in the manipulator example.
We note that \gjkcbf~applies to any pair of geometries satisfying Assumption~\ref{assump: conv}, including superquadrics \cite{Zhao.etal.arXiv26}. Here, we use superellipsoids, a generalization of ellipsoids, to model strictly convex bodies. Each superellipsoid is defined by semi-axes ${\bm{e}=(e_1,e_2,e_3)}$ and shape parameter $\rho$ as ${\mathcal{E}=\{\bm{y}\in\mathbb{R}^3:\sum_{i=1}^{3}\left|y_i/e_i\right|^\rho\leq1\}}$. $\mathcal{E}$ is strictly convex for $\rho>1$ and reduces to an ellipsoid when $\rho=2$, with larger $\rho$ producing increasingly box-like shapes. 

\subsection{Comparative study}\label{subsec: simu-2d}
We first compare our method with the supporting-hyperplane-CBF \cite{Funada.etal.TCST24}, 
designed for ellipsoidal robots. We consider six heterogeneous ellipsoidal robots uniformly distributed on a circle of radius ${r_0=10}$~m and indexed counterclockwise from the purple robot on the $x$-axis (Fig.~\ref{fig: cmp}). 
Each robot's goal is the initial position of its diametrically opposite robot. The robots have semi-axis lengths of ${(2.52, 1.26)}$, ${(2.59, 0.56)}$, ${(1.82, 1.40)}$, ${(2.66, 1.54)}$, ${(2.24, 1.19)}$, and ${(2.10, 0.42)}$~m, respectively. All initial orientations are zero, with desired orientations of $40,90,-25,10,0,-70$ degrees, respectively.
For both methods, we set $k_{v,i}=1.2$, $k_{\omega,i}=0.8$, and use $\mathbf{W}=\mathrm{diag}(1,1,5)$ for all robots. For \cite{Funada.etal.TCST24}, the auxiliary variable's 
virtual-input weight (in the CBF-QP) and gain 
are $1$ and ${k_{z,0}=10}$, respectively. We simulate for ${T=8}$~s with ${\Delta t=5}$~ms.

As shown in Fig.~\ref{fig: cmp}, the proposed \gjkcbf~maintains positive pairwise separation for all time, whereas, under the parameters used here, the supporting-hyperplane-CBF results in inter-robot collisions, which might arise from inter-sample effects. The baseline nevertheless 
has lower computation time by avoiding
explicit pairwise distance queries at each control loop. The trade-off is a higher-dimensional local QP due to auxiliary states and controls for the supporting hyperplanes of each pair.
Our custom GJK implementation takes support functions as inputs and returns the witness pair, but is not optimized for computational efficiency. We anticipate that accelerated variants such as GJK++ \cite{Montaut.etal.TRO24} can further reduce the runtime of \gjkcbf.

For the 3D comparison, we consider a ${4\times8\times1}$~m box-shaped robot navigating through a narrow vertical slit formed by two superellipsoidal obstacles (Fig.~\ref{fig: teaser}). The robot moves from ${\bm{p}^0=[0,2.5,0]^\top}$ with
an initial attitude of
${\bm{\eta}^0=[10^\circ,0,0]^\top}$ to ${\bm{p}^d=[20,2,2.8]^\top}$ with zero desired attitude. The obstacles are centered at ${(12,-4,2)}$ and ${(12,5,2)}$ with semi-axes ${(4.5,3.6,7.2)}$ and ${\rho=3}$. The proposed \gjkcbf~uses the original geometries, whereas for the ellipsoid-only baseline  \cite{Funada.etal.TCST24}, we approximate each body by its minimum-volume enclosing ellipsoid.
We set ${k_v=3}$, ${k_\omega=0.8}$, ${\mathbf{W}=\mathrm{diag}(\mathbf{I}_3,3\mathbf{I}_3)}$, ${T=5}$~s, and ${\Delta t=1}$~ms for both methods.
As shown in Fig.~\ref{fig: teaser}, conservative ellipsoidal approximations block direct passage through the slit, forcing a slight detour above the obstacles
(Fig.~\ref{fig: teaser}-Top), whereas \gjkcbf~uses the original convex bodies, allowing direct passage
(Fig.~\ref{fig: teaser}-Bottom). 

\subsection{Multi-robot swap in 3D}\label{subsec: simu-3d}
We next consider four heterogeneous superellipsoidal robots exchanging diagonally opposite poses. Their initial positions are ${[0,6,-1]^\top}$, ${[0,-6,1]^\top}$, ${[15,-6,1]^\top}$, ${[15,6,-1]^\top}$, with corresponding semi-axes ${(3.6,1.8,1.3)}$, ${(1.2,3.7,0.4)}$, ${(3.2,2.0,2.0)}$, ${(2.8,2.9,1.5)}$; shape parameters ${\rho=2,3,6,4}$. All initial attitudes are zero, while the desired roll-pitch-yaw angles in degrees are ${(20,60,-45)}$, ${(90,25,-10)}$, ${(0,-70,-30)}$, ${(-90,20,30)}$, respectively.
We set ${k_v=1.2}$, ${k_\omega=0.8}$, and ${\mathbf{W}=\mathrm{diag}(\mathbf{I}_3,10\mathbf{I}_3)}$. 
We simulate for ${T=6}$~s with ${\Delta t=10}$~ms. As shown in Fig.~\ref{fig: 3d-swap}, all robots safely exchange poses while maintaining positive pairwise CBF values. 
\begin{figure}[htb!]
\captionsetup[subfigure]{labelformat=empty}
\centering
\begin{subfigure}{\linewidth}
    \centering
    \includegraphics[width=.95\linewidth]{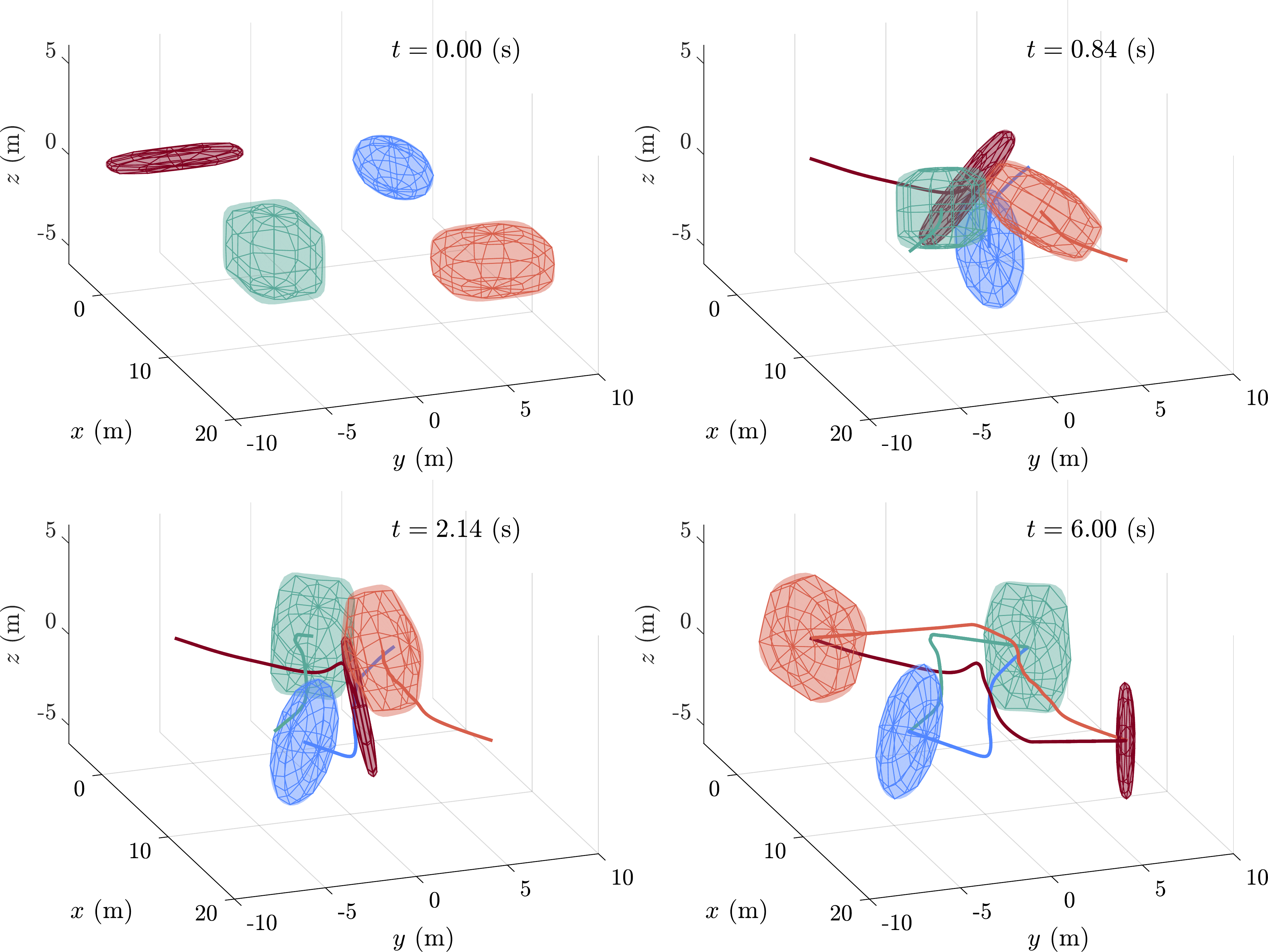}
\end{subfigure}
\\[0.5em]
\begin{subfigure}{\linewidth}
    \centering
    \includegraphics[width=.92\linewidth]{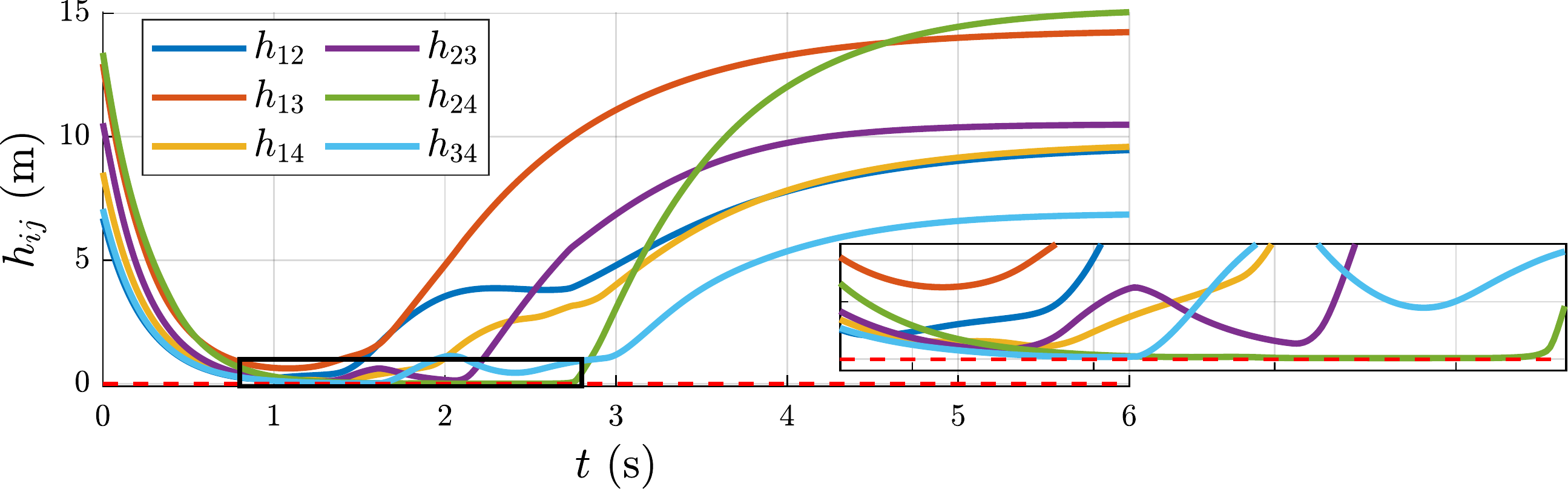}
\end{subfigure}
\caption{\textbf{Top:} Snapshots of four heterogeneous superellipsoidal robots exchanging diagonally opposite poses using the decentralized \gjkcbf~\eqref{eq: decentralized-cbf-qp}. \textbf{Bottom:} The corresponding pairwise CBF values.}
\label{fig: 3d-swap}
\end{figure}

\subsection{6-DOF UR5 manipulator} \label{subsec: simu-arm}
Finally, we apply \gjkcbf~to a velocity-controlled UR5 manipulator modeled using the Robotics System Toolbox.
We use four superellipsoids ${\mathcal{E}_\ell,\ell\in\{1,\ldots,4\}}$ with $\rho=4$ to enclose the upper arm, forearm, and two wrist sections, each rigidly attached to its parent link with body-fixed frame $\mathcal{B}^\ell$.
The state is ${\mathbf{q}=[\mathfrak{q}_1,\ldots,\mathfrak{q}_6]^\top\in\mathbb{R}^6}$, where $\mathfrak{q}_i$ is the $i$-th joint angle. The control input is now the joint velocity ${\mathbf{u}=\dot{\mathbf q}\in\mathbb{R}^6}$. 
The task is to move the end effector (EE) from the start point to the goal point, see Fig.~\ref{fig: simu-ur5}-Top.

A smooth $\nth{5}$-order reference trajectory in joint space is generated as
$\mathbf{q}_{\mathrm{ref}}(t) = \mathbf{q}^0+(\mathbf{q}^d-\mathbf{q}^0)
\left(10\tau^3-15\tau^4+6\tau^5\right)$ with ${\tau=t/T}$. The EE position and velocity reference trajectories, ${\bm{p}_{\mathrm{ref}}}(t)$ and ${{\dot{\bm{p}}_{\mathrm{ref}}}}(t)$, are obtained via forward kinematics. The nominal joint-velocity input is given by
${\overline{\mathbf{u}}=J_v^\dagger\left(K_p(\bm{p}_{\mathrm{ref}}-\bm{p})+\dot{\bm{p}}_{\mathrm{ref}}\right)+\left(\mathbf{I}_6-J_v^\dagger J_v\right)K_q(\mathbf{q}_{\mathrm{ref}}-\mathbf{q})}$,
where ${J_v(\mathbf{q})\in\mathbb{R}^{3\times6}}$ is the translational component of the EE geometric Jacobian, and ${J_v^\dagger= J_v^\top(J_vJ_v^\top+\lambda^2\mathbf{I}_3)^{-1}}$ is its damped pseudoinverse, with ${K_p=4}$, ${K_q=1}$, ${\lambda=0.05}$.

For each link-obstacle pair ${(\ell,o)}$, ${o\in\{1,2\}}$, 
the 
body-velocity Jacobian $\mathbf{J}_{\ell,o}$ given by \eqref{eq: Jaco} 
is evaluated using the origin of the
parent-link frame $\mathcal{B}^\ell$, rather than the center of $\mathcal{E}_\ell$, as the reference point. A separate Jacobian, $J_{\bm{V}_\ell}(\mathbf q)$, is the parent-link geometric Jacobian
and maps the joint velocity $\mathbf{u}$ to the body velocity $\bm{V}_\ell$, i.e., ${\bm{V}_\ell=J_{\bm{V}_\ell}(\mathbf{q})\mathbf{u}}$. Thus, the CBF derivative is computed using the chain rule as ${\dot{h}_{\ell,o}=\mathbf{J}_{\ell,o}\bm{V}_\ell=\mathbf{J}_{\ell,o}J_{\bm{V}_\ell}(\mathbf{q})\mathbf{u}}$. We set ${d_{\mathrm{safe}}=0.03}$~m, ${T=15}$~s, ${\Delta t =50}$~ms, and impose the joint-velocity bounds ${-1.5\leq\dot{\mathfrak{q}}_i\leq1.5}$~rad/s for each joint. As shown in Fig.~\ref{fig: simu-ur5}, without the safety filter, the nominal trajectory (blue) intersects both obstacles, whereas \gjkcbf~modifies the joint velocities and guides the manipulator around the obstacles.

\begin{figure}[htb!]
    \centering
    \includegraphics[width=0.98\linewidth]{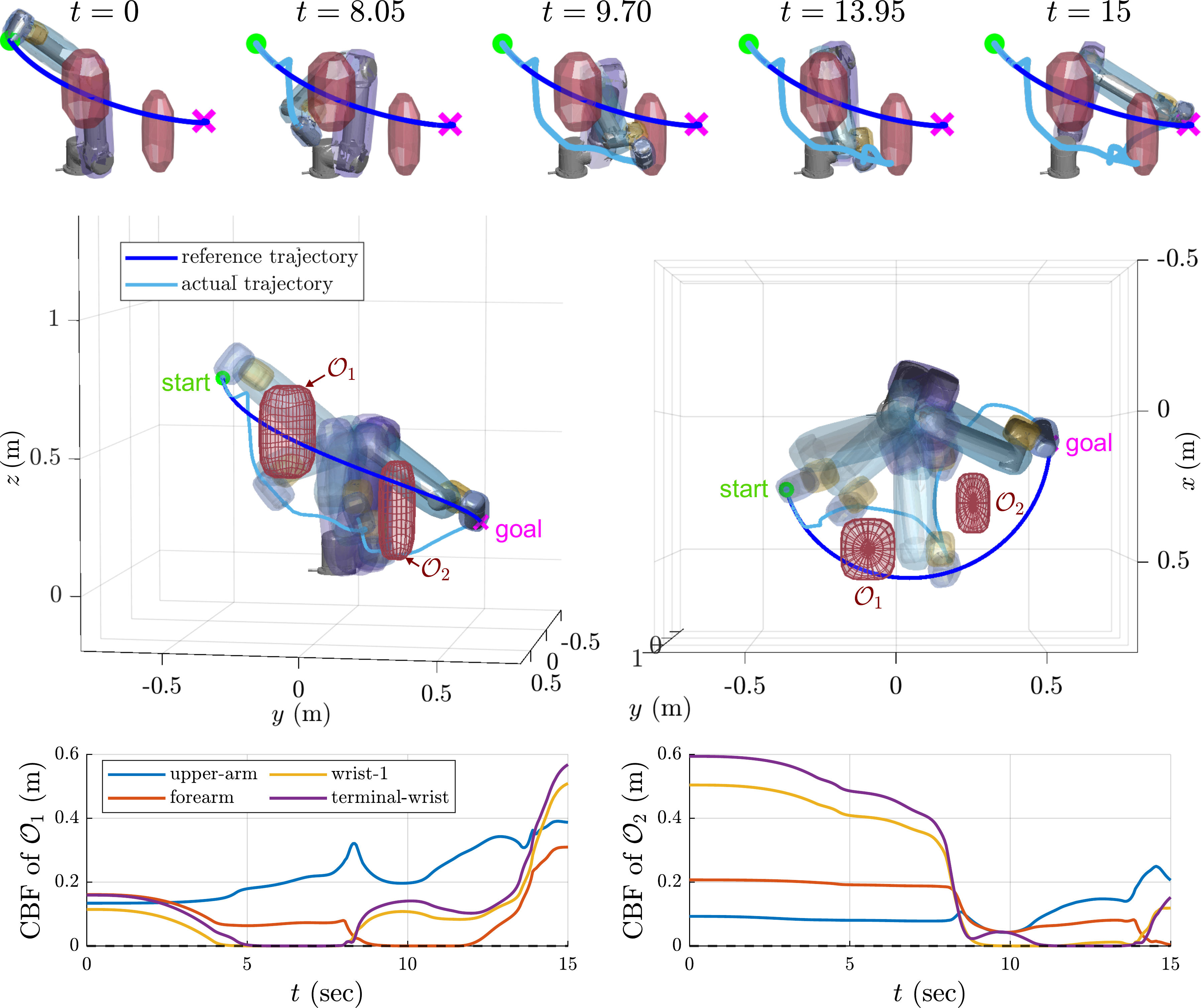}
    \caption{Collision avoidance for a velocity-controlled UR5 manipulator. \textbf{Top:} Side \textbf{(left)} and top \textbf{(right)} views of the motion from the start (green) to the goal (magenta). The blue and cyan curves denote the reference and filtered end-effector trajectories, respectively. \textbf{Bottom:} CBF values for the four links relative to superellipsoidal obstacles $\mathcal{O}_1$ and $\mathcal{O}_2$.
    }
    \label{fig: simu-ur5}
\end{figure}

\section{Conclusions and Future Work}\label{sec: conc}
We present \gjkcbf, a general distance-based CBF formulation for collision avoidance between convex bodies, provided that at least one body in each pair is strictly convex. We show that \gjkcbf~is continuously differentiable and its gradient can be computed solely from the witness pair, without using diffOpt. We validate our method through multi-robot position swapping in both 2D and 3D, a single robot passing through a vertical slit in 3D, and a UR5 manipulator application.
Comparative studies with the supporting-hyperplane-CBF \cite{Funada.etal.TCST24} show that our method is less conservative while maintaining collision-free motion. Future work will relax the strict-convexity assumption using generalized gradients and nonsmooth CBFs for overlapping and other nondifferentiable configurations and extend the framework to higher-relative-degree systems.
\section{Acknowledgment}
The authors used generative AI tools \cite{openai_chatgpt_2026} to polish the grammar, clarify selected derivations, and refine code implementation. All research ideas and technical content were developed by the authors.

\bibliographystyle{IEEEtran}
\bibliography{references} 

@misc{Chen.etal.arXiv26,
      title={Exact Signed-Distance Control Barrier Functions via Minkowski Operations for Safe Navigation among Polytopes}, 
      author={Yi-Hsuan Chen and Shuo Liu and Wei Xiao and Calin Belta and Michael Otte},
      year={2026},
      eprint={2608.02886},
      archivePrefix={arXiv},
      primaryClass={cs.RO},
      url={https://arxiv.org/abs/2608.02886}, 
}

@misc{Chen.etal.CDC26,
      title={Long-Horizon Geometry-Aware Navigation among Polytopes via MILP-MPC and Minkowski-Based CBFs}, 
      author={Yi-Hsuan Chen and Salman Ghori and Ania Adil and Eric Feron and Calin Belta},
      year={2026},
      eprint={2604.00162},
      archivePrefix={arXiv},
      primaryClass={cs.RO},
      url={https://arxiv.org/abs/2604.00162}, 
}

@article{Zhao.etal.arXiv26,
  title={Sq-cbf: Signed distance functions for numerically stable superquadric-based safety filtering},
  author={Zhao, Haocheng and Brunke, Lukas and Lagerquist, Oliver and Zhou, Siqi and Schoellig, Angela P},
  journal={arXiv preprint arXiv:2602.11049},
  year={2026}
}

@book{Li.etal.book94,
  title={A mathematical introduction to robotic manipulation},
  author={Li, Zexiang and Sastry, Shankar and Murray, R},
  year={1994},
  publisher={CRC press Boca Raton, FL}
}

@book{4M,
  title={Computational geometry: algorithms and applications},
  author={De Berg, Mark},
  year={2000},
  publisher={Springer Science \& Business Media}
}

@book{LaValle.Planning.06,
  title={Planning Algorithms},
  author={LaValle, Steven M},
  year={2006},
  publisher={Cambridge University Press}
}

@article{gjk,
  title={A fast procedure for computing the distance between complex objects in three-dimensional space},
  author={Gilbert, Elmer G and Johnson, Daniel W and Keerthi, S Sathiya},
  journal={IEEE Journal on Robotics and Automation},
  volume={4},
  number={2},
  pages={193--203},
  year={1988},
}

@inproceedings{epa,
  title={Proximity queries and penetration depth computation on 3d game objects},
  author={Van Den Bergen, Gino},
  booktitle={Game Developers Conference},
  volume={170},
  pages={209},
  year={2001}
}

@article{Lien.RSS08,
  title={Hybrid motion planning using Minkowski sums},
  author={Lien, Jyh-Ming},
  journal={Proceedings of robotics: science and systems IV},
  year={2008}
}

@article{Wieland.Allgower.IFAC07,
  title={Constructive safety using control barrier functions},
  author={Wieland, Peter and Allg{\"o}wer, Frank},
  journal={IFAC Proceedings Volumes},
  volume={40},
  number={12},
  pages={462--467},
  year={2007},
  publisher={Elsevier}
}

@inproceedings{Ames.etal.ECC19,
  title={Control barrier functions: Theory and applications},
  author={Ames, Aaron D and Coogan, Samuel and Egerstedt, Magnus and Notomista, Gennaro and Sreenath, Koushil and Tabuada, Paulo},
  booktitle={European Control Conference},
  pages={3420--3431},
  year={2019},
}

@inproceedings{Amos.Kolter.ICML17,
  title={{OptNet}: Differentiable optimization as a layer in neural networks},
  author={Amos, Brandon and Kolter, J Zico},
  booktitle={International Conference on Machine Learning},
  pages={136--145},
  year={2017},
  organization={PMLR}
}

@misc{solà.etal.arXiv21,
      title={A micro Lie theory for state estimation in robotics}, 
      author={Joan Solà and Jeremie Deray and Dinesh Atchuthan},
      year={2021},
      eprint={1812.01537},
      archivePrefix={arXiv},
      primaryClass={cs.RO},
      url={https://arxiv.org/abs/1812.01537}, 
}

@misc{openai_chatgpt_2026,
  author       = {{OpenAI}},
  title        = {{C}hat{GPT}},
  year         = {2026},
  howpublished = {[Online]. Available: \url{https://openai.com/chatgpt/}},
}

@article{Funada.etal.TCST24,
  title={Collision avoidance for ellipsoidal rigid bodies with control barrier functions designed from rotating supporting hyperplanes},
  author={Funada, Riku and Nishimoto, Koju and Ibuki, Tatsuya and Sampei, Mitsuji},
  journal={IEEE Transactions on Control Systems Technology},
  volume={33},
  number={1},
  pages={148--164},
  year={2024},
  publisher={IEEE}
}

@article{Thirugnanam.etal.arXiv23,
  title={Control barrier functions for collision avoidance between strongly convex regions},
  author={Thirugnanam, Akshay and Zeng, Jun and Sreenath, Koushil},
  journal={arXiv preprint arXiv:2306.13259},
  year={2023}
}

@ARTICLE{Dai.etal.RAL23,
  author={Dai, Bolun and Khorrambakht, Rooholla and Krishnamurthy, Prashanth and Gonçalves, Vinícius and Tzes, Anthony and Khorrami, Farshad},
  journal={IEEE Robotics and Automation Letters}, 
  title={Safe Navigation and Obstacle Avoidance Using Differentiable Optimization Based Control Barrier Functions}, 
  year={2023},
  volume={8},
  number={9},
  pages={5376-5383},
  doi={10.1109/LRA.2023.3295654}}

@article{Wu.etal.TCB25,
  title={Optimization-free smooth control barrier function for polygonal collision avoidance},
  author={Wu, Shizhen and Fang, Yongchun and Sun, Ning and Lu, Biao and Liang, Xiao and Zhao, Yiming},
  journal={IEEE Transactions on Cybernetics},
  year={2025},
  publisher={IEEE}
}

@INPROCEEDINGS{Molnar.CCTA25,
  author={Molnar, Tamas G.},
  booktitle={2025 IEEE Conference on Control Technology and Applications (CCTA)}, 
  title={Navigating Polytopes with Safety: A Control Barrier Function Approach}, 
  year={2025},
  volume={},
  number={},
  pages={179-184},
  doi={10.1109/CCTA53793.2025.11151473}}

@INPROCEEDINGS{Usevitch.Sahleen.ACC25,
  author={Usevitch, James and Sahleen, Jackson},
  booktitle={2025 American Control Conference (ACC)}, 
  title={Safety for Time-Varying Parameterized Sets Using Control Barrier Function Methods}, 
  year={2025},
  volume={},
  number={},
  pages={267-274},
  doi={10.23919/ACC63710.2025.11107722}}

@article{Matias.Silvestre.arXiv26,
  title={Safe Navigation under Uncertain Obstacle Dynamics using Control Barrier Functions and Constrained Convex Generators},
  author={Matias, Hugo and Silvestre, Daniel},
  journal={arXiv preprint arXiv:2601.07715},
  year={2026}
}

@ARTICLE{Fernandez.etal.RAL26,
  author={Fernandez, Louis and Sutjipto, Sheila and Moreno, Victor Hernandez and Tu, Anh Minh and Carmichael, Marc G.},
  journal={IEEE Robotics and Automation Letters}, 
  title={Differentiable Obstacle Avoidance Framework for Robot Safety}, 
  year={2026},
  volume={11},
  number={4},
  pages={5158-5165},
  doi={10.1109/LRA.2026.3664595}}

@article{Jo.etal.26,
  title={Geometry-Aware Control Barrier Functions for Collision Avoidance via Bernstein Polynomial Approximations},
  author={Jo, Siwon and Zhang, Yanze and Yang, Yupeng and Luo, Wenhao},
  journal={arXiv preprint arXiv:2605.30696},
  year={2026}
}

@article{Guo.etal.IoT26,
  title={An Improved Differentiable Optimization-Based Control Barrier Function Scheme to Reactive Motion Planning for Robots},
  author={Guo, Zhenchuan and Chen, Yukai and Wei, Yanling and Nie, Xiaokai},
  journal={IEEE Internet of Things Journal},
  year={2026},
  publisher={IEEE}
}

@INPROCEEDINGS{Wang.etal.CCDC25,
  author={Wang, Weijia and Meng, Tao and Wang, Kun and Lei, Jiakun},
  booktitle={2025 37th Chinese Control and Decision Conference (CCDC)}, 
  title={Geometric Control Barrier Function: Extensions on the Lie Group SE(3)}, 
  year={2025},
  volume={},
  number={},
  pages={4773-4778},
  doi={10.1109/CCDC65474.2025.11090631}}

@article{Lin.Egerstedt.arXiv26,
  title={Stochastic control barrier functions under state estimation: From Euclidean space to lie groups},
  author={Lin, Ruoyu and Egerstedt, Magnus},
  journal={arXiv preprint arXiv:2601.16198},
  year={2026}
}

@INPROCEEDINGS{De.etal.ACC26,
  author={De Sa, Massimiliano and Ong, Pio and Ames, Aaron D.},
  booktitle={2026 American Control Conference (ACC)}, 
  title={From Bundles to Backstepping: Geometric Control Barrier Functions for Safety-Critical Control on Manifolds}, 
  year={2026},
  volume={},
  number={},
  pages={1928-1935},
  doi={}}

@ARTICLE{Choset.Lee.TRM01,
  author={Choset, H. and Ji Yeong Lee},
  journal={IEEE Transactions on Robotics and Automation}, 
  title={Sensor-based construction of a retract-like structure for a planar rod robot}, 
  year={2001},
  volume={17},
  number={4},
  pages={435-449},
  doi={10.1109/70.954756}}

@article{Lee.Choset.IJRR05,
  title={Sensor-based planning for a rod-shaped robot in three dimensions: Piecewise retracts of R3$\times$ S2},
  author={Lee, Ji Yeong and Choset, Howie},
  journal={The International Journal of Robotics Research},
  volume={24},
  number={5},
  pages={343--383},
  year={2005},
  publisher={SAGE Publications}
}

@article{Schulman.etal.IJRR14,
  title={Motion planning with sequential convex optimization and convex collision checking},
  author={Schulman, John and Duan, Yan and Ho, Jonathan and Lee, Alex and Awwal, Ibrahim and Bradlow, Henry and Pan, Jia and Patil, Sachin and Goldberg, Ken and Abbeel, Pieter},
  journal={The International Journal of Robotics Research},
  volume={33},
  number={9},
  pages={1251--1270},
  year={2014},
  publisher={Sage Publications Sage UK: London, England}
}

@article{lozano-perez_spatial_1983,
    title = {Spatial {Planning}: {A} {Configuration} {Space} {Approach}},
    volume = {C-32},
    issn = {1557-9956},
    shorttitle = {Spatial {Planning}},
    url = {https://ieeexplore.ieee.org/document/1676196/},
    doi = {10.1109/TC.1983.1676196},
    number = {2},
    urldate = {2026-04-16},
    journal = {IEEE Transactions on Computers},
    author = {{Lozano-Perez}},
    month = feb,
    year = {1983},
    pages = {108--120},
}

@article{Escande.etal.TRO14,
  title={A strictly convex hull for computing proximity distances with continuous gradients},
  author={Escande, Adrien and Miossec, Sylvain and Benallegue, Mehdi and Kheddar, Abderrahmane},
  journal={IEEE Transactions on Robotics},
  volume={30},
  number={3},
  pages={666--678},
  year={2014},
  publisher={IEEE}
}

@ARTICLE{Singletary.etal.LCSS22,
  author={Singletary, Andrew and Kolathaya, Shishir and Ames, Aaron D.},
  journal={IEEE Control Systems Letters}, 
  title={Safety-Critical Kinematic Control of Robotic Systems}, 
  year={2022},
  volume={6},
  number={},
  pages={139-144},
  doi={10.1109/LCSYS.2021.3050609}}

@ARTICLE{Xiao.Belta.TAC22,
  author={Xiao, Wei and Belta, Calin},
  journal={IEEE Transactions on Automatic Control}, 
  title={High-Order Control Barrier Functions}, 
  year={2022},
  volume={67},
  number={7},
  pages={3655-3662},
  doi={10.1109/TAC.2021.3105491}}

@ARTICLE{Wang.etal.TRO17,
  author={Wang, Li and Ames, Aaron D. and Egerstedt, Magnus},
  journal={IEEE Transactions on Robotics}, 
  title={Safety Barrier Certificates for Collisions-Free Multirobot Systems}, 
  year={2017},
  volume={33},
  number={3},
  pages={661-674},
  doi={10.1109/TRO.2017.2659727}}

@ARTICLE{Montaut.etal.TRO24,
  author={Montaut, Louis and Le Lidec, Quentin and Petrik, Vladimir and Sivic, Josef and Carpentier, Justin},
  journal={IEEE Transactions on Robotics}, 
  title={GJK++: Leveraging Acceleration Methods for Faster Collision Detection}, 
  year={2024},
  volume={40},
  number={},
  pages={2564-2581},
  doi={10.1109/TRO.2024.3386370}}

@inproceedings{Zimmermann.etal.IROS22,
  title={Differentiable collision avoidance using collision primitives},
  author={Zimmermann, Simon and Busenhart, Matthias and Huber, Simon and Poranne, Roi and Coros, Stelian},
  booktitle={2022 IEEE/RSJ International Conference on Intelligent Robots and Systems (IROS)},
  pages={8086--8093},
  year={2022},
  organization={IEEE}
}

@inproceedings{Li.etal.IROS23,
  title={Frogger: Fast robust grasp generation via the min-weight metric},
  author={Li, Albert H and Culbertson, Preston and Burdick, Joel W and Ames, Aaron D},
  booktitle={2023 IEEE/RSJ International Conference on Intelligent Robots and Systems (IROS)},
  pages={6809--6816},
  year={2023},
  organization={IEEE}
}

@inproceedings{Montaut.etal.ICRA23,
  title={Differentiable collision detection: a randomized smoothing approach},
  author={Montaut, Louis and Le Lidec, Quentin and Bambade, Antoine and Petrik, Vladimir and Sivic, Josef and Carpentier, Justin},
  booktitle={2023 IEEE International Conference on Robotics and Automation (ICRA)},
  pages={3240--3246},
  year={2023},
  organization={IEEE}
}

\end{document}